%% file: root.tex
\documentclass[letterpaper, 10 pt, conference]{ieeeconf} 
\usepackage[ruled,vlined,noend]{algorithm2e}
\usepackage{algcompatible}

\SetCommentSty{mycommentfont}
 \IEEEoverridecommandlockouts  

\usepackage{cite}
\usepackage{amsmath,amssymb,amsfonts}
\usepackage{graphicx}
\usepackage{textcomp}
\usepackage{times}
\usepackage{graphicx}        
\usepackage{multicol}        
\usepackage[bottom]{footmisc}

\usepackage{amsmath}
\usepackage{amsthm}
\usepackage{amssymb}
\usepackage{mathtools}
\usepackage{mathrsfs}
\usepackage[noend]{algpseudocode}
\let\labelindent\relax
\usepackage{enumitem}
\usepackage{multirow}
\usepackage{enumerate}
\usepackage{tabularx}
\usepackage[dvipsnames]{xcolor}
\usepackage{graphicx} 
\graphicspath{{./fig/}}
\usepackage{subcaption}
\usepackage{booktabs}
\usepackage{comment}
\usepackage{float}
\usepackage[breaklinks=true, colorlinks, bookmarks=true]{hyperref}

\usepackage{tikz}
\usepackage{pgfplots}
\pgfplotsset{compat=1.9}

\newtheorem{theorem}{Theorem}[section]

\newtheorem{proposition}[theorem]{Proposition}

\newcommand{\oprocendsymbol}{\hbox{$\bullet$}}
\newcommand{\oprocend}{\relax\ifmmode\else\unskip\hfill\fi\oprocendsymbol}

\newcommand{\norm}[1]{\left\lVert #1 \right\rVert}

\newcommand{\stkout}[1]{\ifmmode\text{\sout{\ensuremath{#1}}}\else\sout{#1}\fi}

\definecolor{paleGreen}{rgb}{.3, .7, .3}

\renewcommand{\theenumi}{(\roman{enumi})}

\newcommand{\x}[1]{$#1\times$}

\definecolor{rrtc}{RGB}{34,34,34}
\definecolor{cbfrrtc}{RGB}{0,114,178}
\definecolor{skipqp}{RGB}{213,94,0}
\definecolor{secondorder}{RGB}{204,121,167}
\definecolor{fixedstep}{RGB}{0,158,115}
\definecolor{firstorder}{RGB}{86,180,233}
\definecolor{ours}{RGB}{122,62,157}

\input{cls/sym}

\allowdisplaybreaks

\title{\LARGE \bf FLASH-CBF-RRT: \underline{F}ast \underline{L}ipschitz-\underline{A}daptive \underline{S}afe \underline{H}orizon CBF Steering for Certified Motion Planning}
\author{%
Mani Amani\textsuperscript{1,2,*},
Hesam Mojtahedi\textsuperscript{1,2,*},
Reza Akhavian\textsuperscript{1},
Nikolay Atanasov\textsuperscript{2},
and Ki Myung Brian Lee\textsuperscript{2}
\thanks{\textsuperscript{*} Mani Amani and Hesam Mojtahedi contributed equally.}
\thanks{\textsuperscript{1}Department of Civil, Construction, and Environmental Engineering,
San Diego State University, San Diego, CA, USA;
\textsuperscript{2}Department of Electrical and Computer Engineering,
University of California San Diego, La Jolla, CA, USA.}
\thanks{This work was supported by the National Science Foundation (NSF) CAREER Award (grant No. 2047138) and the Technology Innovation Program (20018112, Development of autonomous manipulation and gripping technology using imitation learning based on visual and tactile sensing) funded by the Ministry of Trade, Industry \& Energy (MOTIE), Korea.}
}
\begin{document}
\maketitle
\thispagestyle{empty}
\pagestyle{empty}

\begin{abstract} 
Fast motion planning requires both sample-efficient exploration and computationally efficient local steering. 
Conventional straight-line steering is inexpensive but wastes samples by rejecting colliding extensions. 
Quadratic program controllers with control barrier function (CBF) constraints can be used to achieve effective local steering but can increase the planning time due to the expensive sequential evaluations.
We develop a method to accelerate CBF steering by computing a time horizon within which CBF-filtered velocity may be applied constantly using closed-form kinematic bounds on collision-clearance. Our method, FLASH-CBF-RRT, significantly reduces the number of evaluations compared to prior CBF-RRT approaches because the planner can take adaptive steering step sizes that are shorter near obstacles and longer in open space.
We implement this steering procedure as an extension of RRT-Connect in OMPL. Across 689 MotionBenchMaker problems, FLASH-CBF-RRT matches RRT-Connect's 100\% success rate while reducing median planning time by 3.82 times, and is 3.88 times faster than fixed-step CBF-RRT. Hardware experiments on a 17-DoF mobile manipulator demonstrate 2.2 times faster planning and 4.2 times fewer evaluations. Videos and results are available at: \url{https://mamani828.github.io/FLASH-CBF-RRT/}.
\end{abstract}

\input{sections/introduction}
\input{sections/related_works}
\input{sections/problem_formulation}
\input{sections/methodology.tex}

\input{sections/experiments}
\input{sections/conclusion}
\bibliographystyle{IEEEtran}
\bibliography{references}
\end{document}

%% file: cls/sym.tex
\newcommand{\bfp}{\mathbf{p}}
\newcommand{\bfq}{\mathbf{q}}
\newcommand{\bfr}{\mathbf{r}}

\newcommand{\bfu}{\mathbf{u}}
\newcommand{\bfv}{\mathbf{v}}

\newcommand{\bfz}{\mathbf{z}}

\newcommand{\bfomega}{\boldsymbol{\omega}}



%% file: sections/introduction.tex
\section{Introduction}
\label{sec:intro}

Sampling-based motion planners such as rapidly exploring random trees (RRT)~\cite{lavalle1998rapidly}, RRT*~\cite{karaman2011sampling}, RRT-Connect (RRT-C)~\cite{kuffner2000rrt}, and probabilistic roadmap (PRM)~\cite{kavraki1996probabilistic} are widely used for high-dimensional robot motion planning. Their efficiency depends critically on local steering. Conventional steering is inexpensive but follows straight configuration-space paths, rejecting extensions that encounter obstacles. In cluttered environments, this propose-and-reject process wastes samples and limits expansion progress.

Control barrier functions (CBFs)~\cite{ames2017control}
can serve as an alternative steering mechanism because they provide local collision avoidance.
CBF controllers use geometric information such as a signed distance field (SDF)~\cite{kernelSDF} and its gradients to impose safety constraints on control inputs generated as solutions of quadratic programs (QPs). 
Rather than rejecting colliding extensions, CBF controllers can modify the steering command to avoid obstacles, as demonstrated in \cite{yang2019sampling,ahmad2022adaptive,yu2024eff}.
This can improve the sample efficiency of motion planning as samples near collision can be retained in the planning graph.

This advantage, however, comes with significant computational overhead \cite{yang2025lqrcbfrrt}. For robot manipulators, each CBF steering step requires additional computation of forward kinematics (FK) Jacobians and a QP solution.
Consequently, although CBF steering can produce better graph extensions, the computation can still be slower than conventional geometric planning in overall planning time.

\begin{figure}[t]
    \centering
    \includegraphics[width=\linewidth]{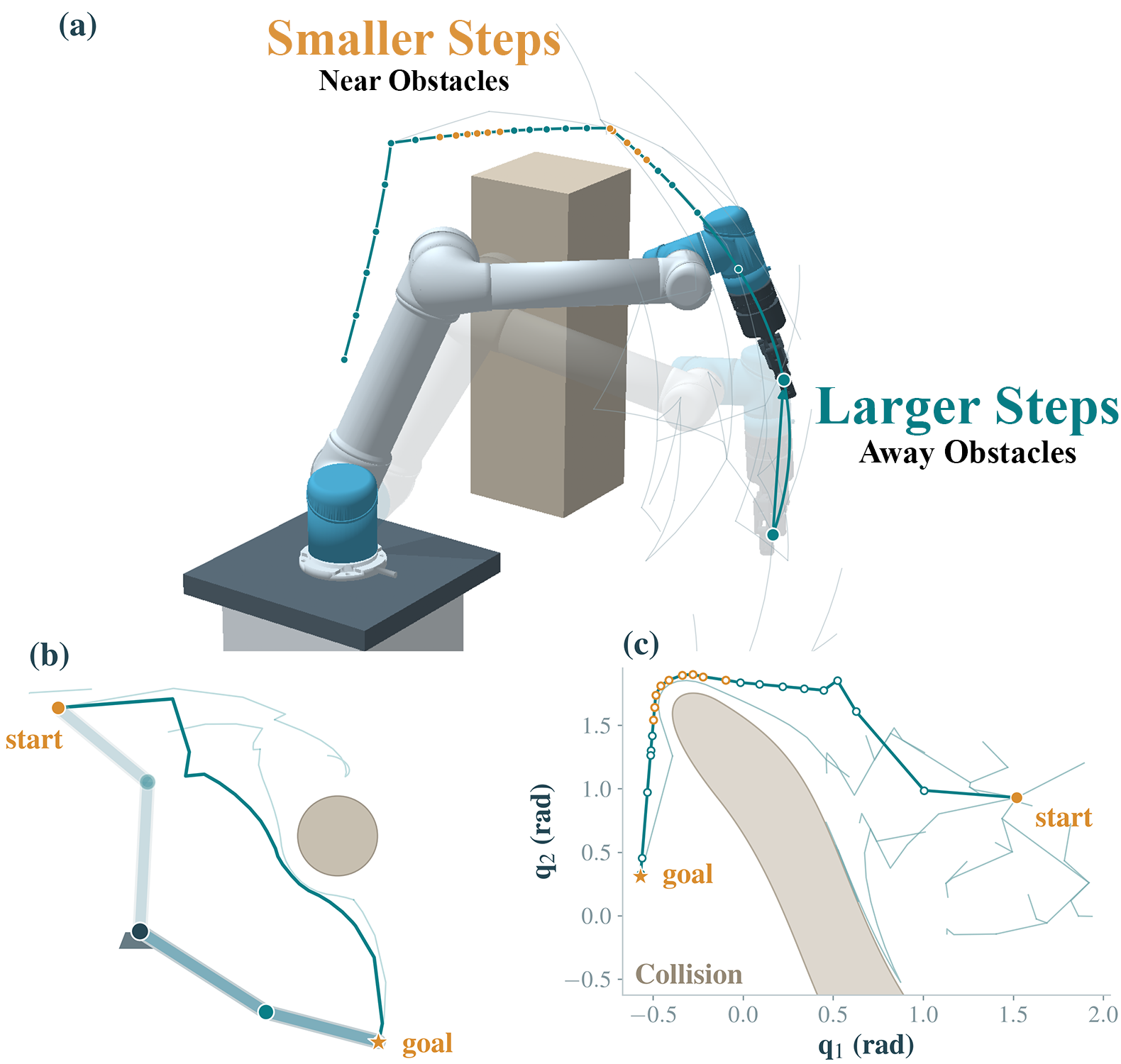}
    \caption{\small \textbf{FLASH-CBF-RRT.} (a) A CBF safety filter steers tree extensions around obstacles instead of rejecting them. Each filtered control is held for a \emph{certified} duration, computed in closed form as the time at which a kinematic bound on link displacement under constant joint velocity would exhaust the current collision clearance. (b)--(c) Workspace- and configuration-space views.}
    \label{fig:mainFig}
\end{figure}

We develop \textbf{F}ast \textbf{L}ipschitz-\textbf{A}daptive \textbf{S}afe \textbf{H}orizon (FLASH)-CBF-RRT, a sampling-based planner designed to provide two advantages: \emph{sample-efficient exploration} through CBF steering and \emph{low computational overhead} through certified adaptive propagation of the CBF controls, illustrated in Fig.~\ref{fig:mainFig}. This certified adaptive propagation step builds on a key observation that a safe control input need not be recomputed at every small integration step but remains valid for a certain safe horizon. We adaptively compute the safe horizon in closed form by bounding the variation in collision clearance under constant joint velocity through Lipschitz continuity. 
The clearance variation bound is obtained from a new closed-form bound on the robot link displacements throughout a constant joint-velocity motion.
The same bounds are used to reduce the individual QP sizes by removing provably inactive CBF constraints. Because the bounds are certified, FLASH-CBF-RRT inherits \emph{continuous motion safety} in addition to faster planning time.

We evaluate FLASH-CBF-RRT on UR5 in seven simulated MotionBenchMaker~\cite{Chamzas_2021} environments, and on a real Reachy2 robot where we plan and execute motions using an SDF representation~\cite{kernelSDF} estimated using an onboard RGB-D camera.
The simulated experiments show that FLASH-CBF-RRT plans $3.82\times$ faster than standard RRT-C \cite{kuffner2000rrt} and $20.83 \times$ faster than a prior CBF-RRT baseline~\cite{yang2025lqrcbfrrt},
which proposes to stop the steer without solving the QP when CBF constraints are active.
Ablation experiments show that our certified adaptive propagation method reduces median planning time by $3.75\times$ and CBF-QP solves
by $3.76\times$.

Our contributions are summarized as follows.
\begin{itemize}
    \item We develop FLASH-CBF-RRT, a sampling-based planner that uses CBF-based steering to improve the effectiveness of planning graph expansions.

    \item We derive a new certified adaptive propagation bound
    that determines how long a filtered control may be safely propagated, allowing FLASH-CBF-RRT to make long excursions with certified safety without intermediate geometric evaluations or CBF-QP solves. This new bound substantially reduces the computational overhead of CBF steering, as well as the number of constraints for the CBF-QP.

    \item We provide an efficient implementation of FLASH-CBF-RRT within OMPL \cite{sucan2012the-open-motion-planning-library} and evaluate it on UR5 simulation benchmarks and a 17-DoF Reachy2 mobile manipulator with real RGB-D data. The results demonstrate faster median planning time than both geometric RRT-C and CBF-RRT without certified adaptive propagation, at an equal success rate. 
\end{itemize}


%% file: sections/related_works.tex
\section{Related Work}
\label{sec:related}

The computational efficiency of sampling-based planners such as RRT~\cite{lavalle1998rapidly} and RRT-C~\cite{kuffner2000rrt} strongly depends on the time spent in collision-checking the candidate configurations. 
Thus, fast planning can be achieved by biasing the samples or speeding up each collision check. 
Approaches such as \cite{thomason2024vamp,curobo} speed up collision checking by exploiting specialized data structures such as \cite{ramsey2024capt} that enable parallel computation.
Here, we also consider speeding up individual geometric operations but do so through mathematical reformulations that are independent of compute hardware. Our approach reduces the overall \emph{number} of geometric operations by biasing the samples toward free space with adaptive steering step sizes.

Many prior approaches also consider biasing the samples toward free space. Classical approaches such as~\cite{maprm} and~\cite{amato98obprm} pre-process the obstacle geometry into discrete primitives that allow generating desirable samples. These approaches require an explicit representation of the obstacles, which is difficult to obtain from real sensor data. 
Instead, other approaches rely on continuous or sensor-driven obstacle representations \cite{Oleynikova_2017} and bias candidate edges via artificial potential fields \cite{qureshi2016potential}, local trajectory optimization \cite{choudhury2016rabit,kim2018dancingprm}, or reinforcement learning~\cite{chiang2019rlrrt}.
However, these methods either require costly trajectory optimization for each candidate motion~\cite{chiang2019rlrrt,kim2018dancingprm}, or do not certify safety of the local connections \cite{qureshi2016potential,chiang2019rlrrt}.
A promising recent direction is to combine sampling-based planners with CBF-based local controllers~\cite{ames2017control} that guarantee safety in the limit of small time-steps.  CBFs allow local collision avoidance by first representing the free space as the super-level set of a function, such as SDFs \cite{kernelSDF}, and imposing a constraint on the rate of change of the function over the controlled trajectory, which is linear in the control action. The combination of CBFs with sampling-based planning was first considered in~\cite{yang2019sampling}. 
Subsequent work improves the computation time, by biasing samples using the cross-entropy method \cite{ahmad2022adaptive} or stopping without solving the QP when the nominal control violates the CBF constraint \cite{yang2025lqrcbfrrt}.
The work in \cite{yu2024eff} extends CBF-RRT to manipulators by learning a neural CBF candidate that takes perception information as input. 
These methods achieve progress around obstacles but repeated barrier evaluations and QP solves introduce computational overhead. 
In comparison, our work focuses on speeding up CBF-RRT by deriving certified adaptive step sizes within which QP solves may be skipped, and reducing the number of CBF constraints at each configuration. 

Adaptive step sizes have been considered in the context of efficient collision checking.
Early work in~\cite{schwarzer2005adaptive} uses adaptive step sizes derived from clearance information to expedite collision checking of a straight line in configuration space.
Similarly, \cite{bialkowski2016efficient} uses clearance information to form reusable collision-free neighborhoods for collision checking. Similar benefits are observed in~\cite{ademovic2016expanded,lee2024safebubblecovermotion}.
Approaches such as~\cite{arslan2017sensory,huh2022probabilisticallysafecorridor,werner2025gpu} construct a locally safe polytopic corridor based on collision samples.  
Here, we use clearance information, and derive a new, tighter bound along a constant configuration velocity to reduce QP solves.
The adaptive step size can be viewed from the perspective of sampled-data and self-triggered CBFs~\cite{Breeden_2022,tan2024zeroorder,yang2019},
which bound the interval between control updates to ensure continuous safety throughout the motion. However, no prior method has computed an explicit safe hold-time by exploiting both the kinematic structure of an articulated manipulator and the structure of an SDF-based safety constraint. We derive such a hold-time for steering, during which the filtered control can be applied without reevaluating the barriers or resolving the CBF-QP.

\begin{figure*}[t]
    \centering
    \includegraphics[width=\textwidth]{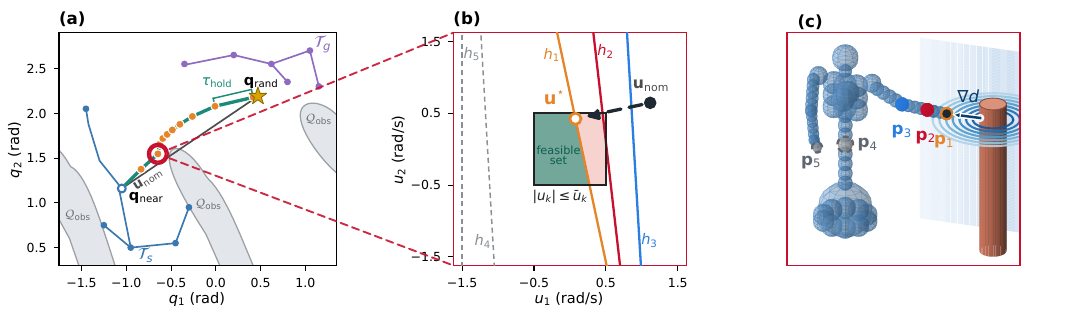}
  \caption{\small \textbf{Certified CBF steering.}
    (a)~One extension from $\bfq_{\mathrm{near}}$ toward $\bfq_{\mathrm{rand}}$. At each orange dot, the safety filter~\eqref{eq:cbf-steer-qp} computes a control held for the certified duration $\tau_{\mathrm{hold}}$, yielding longer segments in open space and shorter segments near obstacles.
    (b)~Control-space view at the circled configuration. Solid lines denote retained CBF constraints; the green polygon is the feasible set within the joint-velocity limits, onto which $\bfu_{\mathrm{nom}}$ is projected to obtain $\bfu^\star$. Dashed lines denote constraints ($h_4$, $h_5$) certified redundant by Theorem~\ref{thm:cbf-row-redundancy}.
    (c)~Workspace view with a sphere-approximated mobile manipulator and an SDF cross-section. SDF queries at sphere centers provide distances and gradients for the safety filter. Spheres 1--3 generate the retained constraints in (b); spheres 4 and 5 generate the screened constraints.}
    \label{fig:flash_method}
\end{figure*}

%% file: sections/problem_formulation.tex
\section{Problem Formulation}
\label{sec:problem}


We consider a manipulator robot with $n$ joints moving in a static workspace $\mathcal{W} \subset \mathbb{R}^3$ that contains obstacles $\mathcal{O} \subset \mathcal{W}$. 
The robot has configuration $\bfq\in\mathcal C$ and velocity input
$\bfu\in\mathbb R^n$, so that $\dot{\bfq} = \bfu$.
The robot's geometry is approximated by spheres indexed by $\mathcal S=\{1,\ldots,N\}$, which can be obtained through tools such as FOAM \cite{foam}.
Each sphere $s \in \mathcal S$ has radius $r_s>0$, and its world-frame center at configuration $\bfq$ is given by the FK map $\bfp_s:\mathcal C\rightarrow\mathbb R^3$. 

We assume that we are given an SDF representation $d:\mathbb{R}^3\rightarrow\mathbb{R}$ of the obstacles $\mathcal{O}$, defined as the signed distance to the closest obstacle:
\begin{equation}\label{eq:sdf_def}
    d(\bfp) = \begin{cases}
        \phantom{-}\min_{\bfp' \in \partial \mathcal{O}} \|\bfp - \bfp'\|, & \bfp \notin \mathcal{O}, \\
        -\min_{\bfp' \in \partial \mathcal{O}} \|\bfp - \bfp'\|, & \bfp \in \mathcal{O}.
    \end{cases}
\end{equation}

An SDF is 1-Lipschitz, i.e., $|d(\bfp') - d(\bfp)| \leq \| \bfp' - \bfp \|$. Whenever differentiable, its gradient satisfies $\| \nabla d(\bfp) \| = 1$. We assume that we have a differentiable SDF map from sensor data that satisfies the 1-Lipschitz property. This can be achieved using tools such as \cite{1-lipschitzneural,kernelSDF}.

Let $\mathcal P\subseteq\{(a,b)\in\mathcal S\times\mathcal S:a<b\}$ denote sphere pairs on non-adjacent robot links that may self-collide. A configuration $\bfq$ is collision-free if
\begin{equation}\label{eq:coll_free}
\begin{aligned}
d(\bfp_s(\bfq)) &\geq r_s, && \forall s\in\mathcal S,\\
\norm{\bfp_a(\bfq)-\bfp_b(\bfq)} & \geq r_a+r_b,
&& \forall (a,b)\in\mathcal P.
\end{aligned}
\end{equation}
Given two configurations $\bfq_{\mathrm{start}},\bfq_{\mathrm{goal}}$, we aim to find a continuous path $\pi:[0,1]\rightarrow\mathcal{C}$ that satisfies the collision constraints~\eqref{eq:coll_free} and the boundary conditions $\pi(0)=\bfq_{\mathrm{start}}$ and $\pi(1)=\bfq_{\mathrm{goal}}$.
As shown in Fig.~\ref{fig:flash_method}, our approach utilizes the SDF $d(\bfp)$ and its gradients $\nabla d(\bfp)$ to: a) find collision-avoiding velocities to efficiently explore the free space; and b) compute certified adaptive horizons within which the velocities from a) can be used without recomputing.

%% file: sections/methodology.tex
\section{FLASH-CBF-RRT}
\label{sec:overview}
FLASH-CBF-RRT grows two trees, $\mathcal{T}_s$ from the start and $\mathcal{T}_g$ from the goal. Each iteration samples $\bfq_{\mathrm{rand}}$, finds its nearest node $\bfq_{\mathrm{near}}$ in the active tree, and extends from $\bfq_{\mathrm{near}}$ toward  $\bfq_{\mathrm{rand}}$ with the certified CBF steering of Algorithm~\ref{alg:certified-cbf-steer}, yielding $\bfq_{\mathrm{new}}$. The other tree then steers toward $\bfq_{\mathrm{new}}$ until it is reached or progress stalls, after which the trees swap roles; planning ends once they connect. Unlike the straight-line steering and collision checking of RRT-C, each steering update bends around obstacles using a CBF-QP controller. 
The CBF steering is accelerated by certifying how long a control can be held without re-evaluation (\textsc{FirstOrderBound} and \textsc{SecondOrderBound}, Sec.~\ref{sec:certified-steering-horizon}) and removing redundant constraints (\textsc{Screen}, Sec.~\ref{sec:lipschitz-row-reduction}). 

\begin{algorithm}
\DontPrintSemicolon
\caption{Certified CBF Steering}
\label{alg:certified-cbf-steer}
\KwIn{$\bfq$ and $\bfq_{\mathrm{target}}$}
\KwOut{Steering status and $\bfq_{\mathrm{new}}$}
\While{$\bfq$ has not reached $\bfq_{\mathrm{target}}$}{
    $\bfu_{\mathrm{nom}}\gets K(\bfq_{\mathrm{target}}-\bfq)$\; 
    $\tau_{\mathrm{target}}
\gets
\frac{\lVert \mathbf{q}_{\mathrm{target}}-\mathbf{q}\rVert}
     {\lVert \mathbf{u}_{\mathrm{nom}}\rVert}.$\;
    $\mathcal I_{\mathrm{QP}}\gets\textsc{Screen}(\bfq)$ \tcp*[r]{Thm.~\ref{thm:cbf-row-redundancy}}
    $\bfu^\star\gets\textsc{CBF-QP}(\bfq,\bfu_{\mathrm{nom}},\mathcal I_{\mathrm{QP}})$ \tcp*[r]{Eq.~\eqref{eq:cbf-steer-qp}} \If{CBF-QP infeasible $\lor$ $\max_j |u^\star_j| \le \epsilon_u$} {\Return $(\textsc{Trapped},\bfq)$\;}
    \tcp{Props.~\ref{prop:certified-steering-horizon} and ~\ref{prop:kinematic}}
    \ForEach{$i\in\mathcal I$}{
    $\tau_{i,\mathrm{Lip}}\gets\textsc{FirstOrderBound}(h_i,\bfq,\bfu^\star)$\;
        $\tau_{i,\mathrm{Kin}}\gets\textsc{SecondOrderBound}(h_i,\bfq,\bfu^\star)$\; 
        $\tau_i\gets\max\{\tau_{i,\mathrm{Lip}},\tau_{i,\mathrm{Kin}}\}$\;
    }
     $\tau_\mathrm{hold}\gets\min\{\tau_{\mathrm{target}},\min_{i\in\mathcal I}\tau_i\}$ \tcp*[r]{Eq.~\eqref{eq:tau-hold}}
    \If{$\tau_\mathrm{hold}$ below threshold}{\Return $(\textsc{Advanced},\bfq)$\;}
    $\bfq\gets\bfq+ \tau_\mathrm{hold}\,\bfu^\star$\;
}
    $\bfq_{\mathrm{new}}\gets\bfq$\;
\Return $(\textsc{Reached},\bfq_{\mathrm{new}})$\;
\end{algorithm}



\subsection{CBF-Based Steering}
\label{sec:cbf-steer}
Standard sampling-based planners steer from $\bfq_{\mathrm{near}}$ toward a target $\bfq_{\mathrm{target}}$ along a straight line in configuration space and reject the segment if any intermediate state is in collision. 
We instead \emph{filter} the steering command online. Let $\bfq\in\mathcal C$ be the current configuration during an extension and define a nominal control $\bfu_{\mathrm{nom}}=K(\bfq_{\mathrm{target}}, \bfq)$ given by a controller.

Using the collision model from Sec.~\ref{sec:problem}, define the barriers
\begin{equation}
h_i(\bfq)
:=
\begin{cases}
d(\bfp_i(\bfq))-r_i,
& i\in\mathcal S,\\
\norm{\bfp_a(\bfq)-\bfp_b(\bfq)}\!-\!r_a\!-\!r_b,
& i\!=\!(a,b)\!\in\!\mathcal P\!.
\end{cases}
\label{eq:barriers}
\end{equation}
where $i\in \mathcal I:=\mathcal S\cup\mathcal P$ indexes the sphere--obstacle and self-collision barriers. Rather than
executing $\bfu_{\mathrm{nom}}$ directly, we \emph{filter} it to satisfy the CBF constraint as:
\begin{equation}
\begin{aligned}
    \bfu^\star=\arg\min_{\bfu}\quad&\tfrac{1}{2}\|\bfu-\bfu_{\mathrm{nom}}\|^2\\
    \text{s.t.}\quad&\nabla h_i(\bfq)^\top\bfu\geq-\kappa h_i(\bfq),\quad \forall i\in\mathcal I,\\
    &\bfu_{\min}\leq\bfu\leq\bfu_{\max},
\end{aligned}
\label{eq:cbf-steer-qp}
\end{equation}
where $\kappa>0$ is the CBF gain bounding the admissible rate of barrier decrease. The problem~\eqref{eq:cbf-steer-qp} is a QP, and can be solved using methods such as~\cite{qpmad}.
We integrate the result through the first-order model $\dot\bfq=\bfu^\star$. Collision avoidance is thus enforced during the extension itself rather than checked afterward. Because the CBF-QP acts as a local safety filter, there may exist configurations for which the feasible control is substantially attenuated by the active barrier constraints, resulting in little or no progress toward $\bfq_{\mathrm{target}}$. We detect such cases using a minimum-control threshold; if $\max_j |u^\star_j| \le \epsilon_u$, where $\epsilon_u$ is the termination threshold, the current extension is terminated early. 


\subsection{Certified Adaptive Propagation}
\label{sec:certified-steering-horizon}
While CBF steering can provide sample efficiency, it has to solve a QP at every timestep. Meanwhile, the filtered control $\bfu^\star$ can stay collision free for a longer period while making progress towards the node. We derive a certified adaptive horizon that the filtered control remains collision free, thus skipping QP solves and proximity queries.

Let $\bfq_0$ be the
initial configuration at which the QP~\eqref{eq:cbf-steer-qp} was solved, and let $\tau_{\mathrm{target}}>0$ denote the
maximum propagation duration considered by the current steering update.
Holding $\bfu^\star$ under the first-order model introduced in
Sec.~\ref{sec:cbf-steer}:
\begin{equation}
    \bfq(\tau)=\bfq_0+\tau \bfu^\star,
    \qquad
    0\leq \tau \leq\tau_{\mathrm{target}}.
    \label{eq:held-control-trajectory}
\end{equation}

The exact safe horizon of the held control is
\begin{equation}
    \tau_{\mathrm{safe}}^\star
    :=
    \inf\big\{
        \tau\in[0,\tau_{\mathrm{target}}] \mid
        \ h_i(\bfq(\tau)) < 0,\; \exists i\in\mathcal I
    \big\}.
\label{eq:true-safe-hold-time}
\end{equation}
Here, \(\tau_{\mathrm{target}}\) is the nominal propagation time to the target and caps propagation under $\bfu^\star$ to preserve RRT's local exploration behavior. To bound  $\tau^\star_{\mathrm{safe}}$, we use the fact that all barriers $h_i$ are a form of geometric distance, and are hence 1-Lipschitz in the corresponding workspace sphere center $\bfp_i(\bfq)$. Therefore, the change in any $h_i$ is bounded by the displacement in the sphere center. With $\bfp_{i}(\tau) = \bfp_i(\bfq(\tau))$ as a shorthand:
\begin{equation}\label{eq:displacement}
    d_i(\tau) = \|\bfp_i(\tau)- \bfp_i(0) \|,
\end{equation}
and $h_i(\bfq(\tau)) \geq h_i(\bfq(0)) - d_i(\tau)$ by Lipschitz continuity.
We derive two closed-form upper bounds on the displacement $d_i(\tau)$ to lower bound the barrier value at any $\tau$, and identify safe horizons within which the barrier values remain positive. 

\subsubsection{First-Order Bound}
We first derive a linear upper bound on $d_i(\tau)$ using the Lipschitz continuity of the FK map $\bfp_s(\cdot)$. We use the following upper bound on the Lipschitz constant from~\cite{ademovic2016expanded}. For each joint $j$, let $\ell_j$ denote the length of joint $j$, and for each sphere $s \in \mathcal{S}$, let $m_s$ denote the joint the sphere $s$ belongs to. The Lipschitz constant for each sphere's position $\bfp_s$ with respect to joint $j$ is bounded by:
\begin{equation}
    \Lambda_{sj} = \sum_{k=j}^{m_s} \ell_k,
\end{equation}
so that $d_i(\tau) \leq \sum_{k} \Lambda_{ik} |q_k(\tau) - q_k(0)|$ for $i \in \mathcal{S}$. Intuitively, the constants $\Lambda_{sj}$ corresponds to the largest possible radius of rotation of sphere $s$ around joint $j$'s axis, assuming a fully extended manipulator~\cite{ademovic2016expanded}.
For the barriers $i \in \mathcal{I}$, the Lipschitz constants are:
\begin{equation}
A_{ik}
:=
\begin{cases}
\Lambda_{ik},
& i\in\mathcal S,\\
\Lambda_{ak}+\Lambda_{bk},
& i=(a,b)\in\mathcal P.
\end{cases}
\label{eq:barrier-lipschitz-coefficients}
\end{equation}
which satisfy $h_i(\bfq(\tau)) - h_i(\bfq(0)) \geq -\sum_{k} A_{ik} |q_k(\tau) - q_k(0)|$ for all barriers. This follows from the 1-Lipschitz property of the environment distance field for the environment barriers $i \in \mathcal{S}$, and from the triangle inequality for the self-collision barriers $i \in \mathcal{P}$. 
Using the Lipschitz constants $A_{ik}$, the first-order bound is given as follows:

\begin{proposition}[First-order horizon]
\label{prop:certified-steering-horizon}
For a given $\bfu^\star$, consider the first-order bound: 
\begin{equation}
    \gamma_i^\star
    :=
    \sum_{k=1}^{n}
    A_{i k}|u_k^\star|.
    \label{eq:held-control-decrease-rate}
\end{equation}
Suppose $h_i(\bfq_0)\geq0$ for every $i\in\mathcal I$, and define
\begin{equation}
    \tau_{i,\mathrm{Lip}}
    :=         \frac{h_i(\bfq_0)}{\gamma_i^\star}
    \label{eq:per-barrier-lipschitz-horizon}
\end{equation}
with $+\infty$ for $\gamma_i^\star=0$. Then, 
\begin{equation}
    \tau_{\mathrm{Lip}}
    :=
    \min
    \left\{
        \tau_{\mathrm{target}},
        \min_{i\in\mathcal I}
        \tau_{i,\mathrm{Lip}}
    \right\} \leq \tau_{\mathrm{safe}}^\star
    \label{eq:global-lipschitz-horizon}
\end{equation}
is a certified hold time, i.e., $h_i(\bfq(\tau)) \geq 0$, $\forall \tau \in [0, \tau_{\mathrm{Lip}}]$.
\end{proposition}

\begin{proof}
Substituting $\bfu^\star$, for every $\tau \leq\tau_{\mathrm{Lip}}$ and
$i\in\mathcal I$, $h_{i}(\bfq(\tau)) - h_i(\bfq_0) \geq -\tau \gamma_i^\star$ by Lipschitz continuity. Since $\tau \leq \tau_{\mathrm{Lip}}$, substituting~\eqref{eq:per-barrier-lipschitz-horizon} yields the claim. 
\end{proof}

This bound is conservative for two reasons. First, the
coefficients $\Lambda_{ik}$ are worst-case kinematic bounds over
$\mathcal C$. Second, \eqref{eq:held-control-decrease-rate} assumes
that all joint contributions simultaneously reduce the same
clearance. Consequently, the resulting horizon $\tau_{\mathrm{Lip}}$ may be substantially smaller than $\tau_{\mathrm{safe}}^\star$.

\subsubsection{Second-Order Bound}
We now derive a second-order bound on the displacement. To do so, we invoke 
Taylor's theorem with integral remainder as:
\begin{equation}
    d_i(\tau)
    \le
    \tau\|\dot{\bfp}_i(0)\|_2
    +\int_0^\tau(\tau-s)\|\ddot{\bfp}_i(s)\|\,ds.
    \label{eq:quadratic-displacement}
\end{equation}
The linear term can be computed exactly as $\|\dot{\bfp}_i(0)\| = \| J_s(\bfq_0)\bfu^\star \|$, where $J_s(\bfq)$ is the position Jacobian of sphere $s$ at $\bfq$. We derive a uniform bound on the second order term $\|\ddot{\bfp}_i(s)\|$ over $[0, \tau_{\mathrm{target}}]$ so that~\eqref{eq:quadratic-displacement} can be bounded uniformly. 
The acceleration of the sphere center evolves as:
\begin{equation}\label{eq:acceleration}
    \ddot{\bfp}_s(\tau) \!=\! \sum_{j=1}^{m_s} \dot {\boldsymbol{\omega}}_j(\tau) \times \bfr_{sj}(\tau) + {\boldsymbol{\omega}}_j(\tau) \times \left({\boldsymbol{\omega}}_j(\tau) \times \bfr_{sj}(\tau) \right),
\end{equation}
where $\times$ is the cross product, $\boldsymbol{\omega}_j(\tau)$ is the angular velocity of link $j$ at $\tau$, and $\bfr_{sj}(\tau)$ is the vector between: 1) for $j < m_s$, the origins of joint $j$ and $j+1$ in the world frame, and 2) for $j = m_s$, the origin of joint $j$ and the center $\bfp_{s}$. 

To bound linear acceleration~\eqref{eq:acceleration}, we will bound the angular velocity $\|\boldsymbol \omega_j (\tau) \|$ and angular acceleration $\|\dot{\boldsymbol \omega}_j (\tau) \|$ uniformly over $\tau \in [0, \tau_{\mathrm{target}}]$. 
First, $\bfomega_j = \sum_k u^\star_k \bfz_k(\tau) $, where $\bfz_k(\tau)$ are unit vectors along each joint axis in the world frame. 
Suppose that $\underline c_{km}\le\bfz_k(\tau)^{\top}\bfz_m(\tau)\le\overline c_{km}$ for all $k<m$ on the path and $t\in[0,\tau_{\mathrm{target}}]$, which we will derive soon.
Expanding $\|\bfomega_j\|^2$ gives
\begin{equation}\label{eq:angular-velocity-bound}
\begin{aligned}
    \|\bfomega_j\|^2 &\!\leq\!
    \sum_{k=1}^{j}(u_k^\star)^2
    \!+\!\!\!\!\!\!\sum_{1\le k<m\le j}\!\!
    2\max\{u_k^\star u_m^\star\underline c_{km},
          u_k^\star u_m^\star\overline c_{km}\}. \\
    &:= \Omega_j^2
\end{aligned}
\end{equation}
For angular acceleration, we have $\dot{\bfomega}_j=\sum_{k=1}^{j}u_k^\star(\bfomega_{k-1}\times\bfz_k)$. The cross product term is bounded as $\|\bfomega_{k-1}\times\bfz_k\|^2=\|\bfomega_{k-1}\|^2-(\bfomega_{k-1}^{\mathsf T}\bfz_k)^2$.
Using the bounds $\underline{c}_{km}$ and $\overline{c}_{km}$ on the dot-product terms from before, it can be shown that:
\begin{equation}
\begin{aligned}
    \| \dot{\bfomega}_j \| &\leq \sum_{k=1}^{j}|u_k^\star| \sqrt{\Omega_{k-1}^2- \left( \max\{0,\underline\eta_k,-\overline\eta_k\} \right) ^2}, \\
    &:= B_j,
\end{aligned}
\label{eq:angular-acceleration-bound}
\end{equation}
where  $\underline{\eta}_k :=\sum_{r=1}^{k-1}
    \min\{u_r^\star\underline c_{rk},u_r^\star\overline c_{rk}\}$, and $\overline\eta_k$ sums over the maximum. The angular velocity bound~\eqref{eq:angular-velocity-bound} is understood with $\Omega_0:=0$ and empty sums for $\eta$ are equal to zero.
The radicand is nonnegative because it derives from bounding the dot product $\bfomega_{k-1}^\top \bfz_k$.  
Because the segment lengths $\ell_j = \| \bfr_{sj}(\tau) \|$ are fixed, the acceleration~\eqref{eq:acceleration} can be bounded as: 
\begin{equation}\label{eq:acceleration-bound}
\begin{aligned}
\|\ddot{\bfp}_s(\tau) \| &\leq \sum_{j=1}^{m_s} \| \dot {\boldsymbol{\omega}}_j(\tau) \| \ell_j + \| {\boldsymbol{\omega}}_j(\tau) \|^2 \ell_j \\
&\leq \sum_{j=1}^{m_s} \ell_j (B_j + \Omega_j^2) := C_s.
\end{aligned}
\end{equation}
This gives our second-order bound:
\begin{proposition}[Second-order safe horizon]
\label{prop:kinematic}
Suppose $h_i(\bfq_0)\ge0$ for every $i\in\mathcal I$. Let the per-barrier second-order bound be:
\begin{equation}
    \tau_{i,\mathrm{Kin}}:=
    \frac{2h_i(\bfq_0)}{\|J_i(\bfq_0)\bfu^\star\|+\sqrt{\|J_i(\bfq_0)\bfu^\star\|^2+2C_ih_i(\bfq_0)}},
    \label{eq:quadratic-horizon}
\end{equation}
Then, the second-order horizon defined as:
\begin{equation}
    \tau_{\mathrm{Kin}}
    :=
    \min\left\{
        \tau_{\mathrm{target}},\;
        \min_{i\in\mathcal I}
        \tau_{i,\mathrm{Kin}}
    \right\},
    \label{eq:taukin}
\end{equation}
is a certified horizon. I.e., $h_i(\bfq(\tau)) \geq 0$, $\forall i$, $\forall \tau \in [0, \tau_{\mathrm{Kin}}]$.
\end{proposition}

\begin{proof}
Since $h_i$ is $1$-Lipschitz in $\bfp_i$ for $i\in\mathcal S$, $h_i(\bfq(\tau))\ge h_i(\bfq_0)-\|J_i(\bfq_0) \bfu^\star\|\tau-\frac12C_i\tau^2$ on $[0,\tau_{\mathrm{target}}]$ using the Taylor expansion~\eqref{eq:quadratic-displacement} and the acceleration bound~\eqref{eq:acceleration-bound}. Solving the quadratic yields the result.
\end{proof}



\noindent\textbf{Speed-bound correction:} The second-order horizon~\eqref{eq:taukin} yields generally longer hold-times than the first-order horizon~\eqref{eq:global-lipschitz-horizon}. An illustration of the differences between the first-order and second-order horizons, along with their workspace configurations, can be seen in Fig.~\ref{fig:envelopes}. 

However, the second-order horizon can grow loose because the assumed linear speed grows linearly without bound. Meanwhile, the linear velocity $\dot{\bfp}_i=\sum_j\bfomega_j\times\bfr_{ij}$ also admits a uniform bound $V_i:=\sum_{j=1}^{m_i}\ell_j\Omega_j \geq \| \dot{\bfp}_s(\tau)\|$.
For $C_i > 0$, the assumed speed under the acceleration bound~\eqref{eq:acceleration-bound} exceeds the bound $V_i$ at $T_{i} = \tfrac{V_i - \|J_i(\bfq_0)\bfu^\star\|}{C_i}$, after a displacement $D_i = \tfrac{V_i^2 - \|J_i(\bfq_0)\bfu^\star\|^2}{2C_i}$. Thus, when $\tau_{i,\mathrm{Kin}} > T_i$, the safe horizon can be corrected with a constant-speed model as:
\begin{equation}
    \tau_{i,\textrm{Kin}} = \min\{\tau_{\mathrm{target}}, T_i + \frac{h_i(\bfq_0) - D_i}{V_i} \},
\end{equation}\label{eq:speed-bound-correction}
where division by $V_i=0$ is read as $+\infty$. When $\tau_{i,\mathrm{Kin}}\leq T_i$, we retain the value
from~\eqref{eq:quadratic-horizon}.

\noindent\textbf{Combined hold time:}
Since $\tau_{i,\mathrm{Lip}}$ and $\tau_{i,\mathrm{Kin}}$ certify barrier $i$, we take the larger per barrier and the smallest over barriers:
\begin{equation}
    \tau_i := \max\{\tau_{i,\mathrm{Lip}},\tau_{i,\mathrm{Kin}}\},
    \quad
    \tau_{\mathrm{hold}} := \min\Big\{\tau_{\mathrm{target}},\,\min_{i\in\mathcal I}\tau_i\Big\}.
    \label{eq:tau-hold}
\end{equation}
By Propositions~\ref{prop:certified-steering-horizon} and~\ref{prop:kinematic}, $h_i(\bfq(\tau))\ge0$ for all $i\in\mathcal I$ and $\tau\in[0,\tau_{\mathrm{hold}}]$.



\noindent\textbf{Self-collisions:}
To handle self-collisions between pairs $(a,b)\in\mathcal P$, we work in the moving frame of the lowest common ancestor (LCA) link of the two spheres. The term $\|\dot{\bfp}_i(0)\|$ in~\eqref{eq:quadratic-displacement} is replaced with the initial relative velocity between $a$ and $b$ in the LCA frame. The corresponding acceleration bound~\eqref{eq:acceleration-bound} follows from summing the acceleration bounds of each sphere, computed while treating the LCA as the base link of the manipulator. This construction replicates all our statements for the self-collision barrier.



\noindent\textbf{Bounding the angle between joint axes:}
Let $k<m$, and let $\theta_{km}^0:=\arccos(\bfz_k(0)^{\mathsf T}\bfz_m(0))$.
Only joints strictly between $k$ and $m$ change the angle between the two axes: proximal ($\leq k$) joints rotate both, whereas distal ($\geq m$) joints rotate neither.
Each intermediate joint rotates $\bfz_m$ relative to $\bfz_k$ by at most its own rotation angle from the motion $\bfu^\star$. So, the angle changes by at most $\delta_{km}:=\tau_{\mathrm{target}}\sum_{r=k+1}^{m-1}|u_r^\star|$.
With the lower bound $\theta_{km}^-:=\max\{0,\theta_{km}^0-\delta_{km}\}$ and the upper bound $\theta_{km}^+:=\min\{\pi,\theta_{km}^0+\delta_{km}\}$, the monotonicity of cosine on $[0,\pi]$ gives $\underline c_{km}:=\cos\theta_{km}^+$ and $\overline c_{km}:=\cos\theta_{km}^-$.
Adjacent axes ($m=k+1$) keep constant angles exactly.

\begin{figure}
    \centering
    \includegraphics[width=\linewidth]{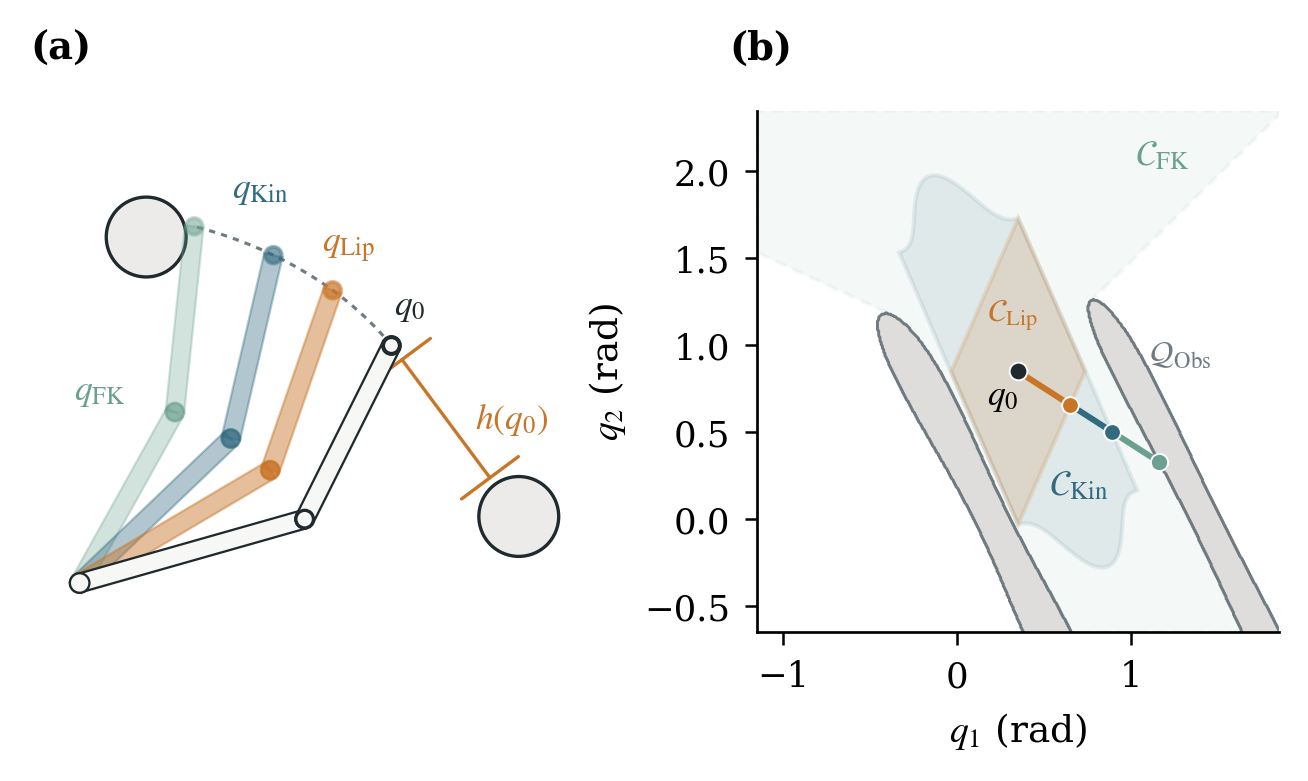}
\caption{\small Workspace and configuration-space view of the adaptive horizon bounds at $\bfq_0$.
  (a) The initial robot configuration, the first-order and second-order
  certified endpoints $\bfq_{\mathrm{Lip}}$ and $\bfq_{\mathrm{Kin}}$, and the
  exact contact configuration $\bfq_{\mathrm{FK}}$ in the workspace. 
  (b) A configuration-space view of (a), with regions showing potential hold times along other directions.}
    \label{fig:envelopes}
\end{figure}
\subsubsection{Holonomic Mobile Bases}
A holonomic mobile base adds additional degrees of freedom in linear velocity $\bfv_{\mathrm{base}} \in \mathbb{R}^2$ and rotation $\omega_{\mathrm{base}} \in \mathbb{R}$. Rotation $\omega_{\mathrm{base}}$ can be added without any modification by treating it as another revolute joint at the mobile base link.

Translation adds to displacement. For the first-order bound, \eqref{eq:held-control-decrease-rate} is modified as $\gamma_i^\star \rightarrow \gamma_i^\star + \|  \bfv_{\mathrm{base}} \|$ for environment barriers \(i\in \mathcal{S}\). For the second-order bound, the initial velocity changes as $\| \dot \bfp_i(0) \| \rightarrow \| \dot \bfp_i(0) + \bfv_{\mathrm{base}} \| $, and the velocity bound as $V_i \rightarrow V_i + \|\bfv_{\mathrm{base}} \|$. 
 

\subsection{Lipschitz QP Row Reduction}
\label{sec:lipschitz-row-reduction}
The Lipschitz coefficients used to certify the adaptive horizon
also bound the decrease of each barrier over the entire admissible
control set. This allows removing CBF constraints that are inactive over admissible velocities: 

\begin{theorem}[Certified CBF constraint removal]
\label{thm:cbf-row-redundancy}
Let $\bar u_k :=
    \max\{|u_{\min,k}|,|u_{\max,k}|\}$ be the maximum joint velocity from~\eqref{eq:cbf-steer-qp}. For each
barrier $i\in\mathcal I$, the worst-case decrease rate is
\begin{equation}
    \bar r_i
    :=
    \sum_{k=1}^{n}A_{i k}\bar u_k.
    \label{eq:max-barrier-decrease}
\end{equation}
Now, for each barrier $i \in \mathcal{I}$, $h_i(\bfq) \geq \frac{\bar r_i}{\kappa}$ implies $\nabla h_i(\bfq)^\top\bfu \geq -\kappa h_i(\bfq)$ for all $\bfu$ satisfying the limit $\bfu_{\min} \leq \bfu \leq \bfu_{\max}$. 
Consequently,
\eqref{eq:cbf-steer-qp} may be solved using only
\begin{equation}
    \mathcal I_{\mathrm{QP}}(\bfq)
    :=
    \left\{
        i\in\mathcal I:
        h_i(\bfq)<\frac{\bar r_i}{\kappa}
    \right\}.
    \label{eq:screened-constraint-set}
\end{equation}
\end{theorem}

\begin{proof}
For every admissible $\bfu$ of~\eqref{eq:cbf-steer-qp}, $|u_k|\leq\bar u_k$, and Lipschitz continuity implies:
\begin{equation}
    \nabla h_i(\bfq)^\top\bfu
    =
    \sum_{k=1}^{n}
    \frac{\partial h_i}{\partial q_k}(\bfq)u_k
    \geq
    -\sum_{k=1}^{n}A_{i k}\bar u_k,
\end{equation}
where the right-hand side is $-\bar{r}_i$.
By precondition, $-\bar r_i\geq-\kappa h_i(\bfq)$, so the CBF constraint for $i$ is
satisfied for all $\bfu$ satisfying $\bfu_{\min} \leq \bfu \leq \bfu_{\max}$. Such constraints therefore
do not affect the feasible set or the solution.
\end{proof}
For a mobile manipulator, Thm.~\ref{thm:cbf-row-redundancy} extends similarly by treating \(\omega_\mathrm{base}\) as an additional revolute joint, and after, for \(i \in \mathcal{S}, \bar r_i \coloneqq \bar v_{\mathrm{base}} + \sum_{k=1}^{n}A_{i k}\bar u_k. \) Self-collision barriers are unaffected since base motion preserves pairwise distances.

Unlike fixed distance thresholds or nearest-obstacle heuristics, our screening is certified over the entire admissible control set $\bfu_{\min} \leq \bfu \leq \bfu_{\max}$. Removing these constraints not only speeds up the QP solve, but also obviates the evaluation of gradients $\nabla h_i(\bfq_0)$.

%% file: sections/experiments.tex
\section{Experiments}

We evaluate all planners on 689 UR5 planning problems from
MotionBenchMaker\cite{Chamzas_2021} across seven cluttered environments: Box, Cage, Small Bookshelf, Tall Bookshelf, Narrow Bookshelf, Tabletop Pick, and Under-Table Pick.

All planners are implemented in C++ using OMPL’s RRT-C implementation. The CBF variants are implemented as a subclass of \texttt{ompl::base::StateSpace}. During each RRT-C extension, this class invokes the CBF filter and records the resulting safe trajectory. The CBF QPs are solved using QPMAD \cite{qpmad} with \(\kappa=0.5\) with a proportional nominal controller. Fixed-step steering uses an integration step of 1~ms. All methods use a goal tolerance of 0.1~rad and a maximum extension range of 2.0~rad. A 10-second timeout was used for all plans. RRT-C uses collision-checking resolution of~0.01 rad along a straight line. All planners run on a computer with a 24-core Intel Core Ultra 9 285K CPU and 64~GiB of RAM.

We compare \textbf{\textcolor{ours}{our}} (FLASH-CBF-RRT-C) method against 3 baselines and 3 ablations:
\par
    \textbf{Baselines}
    \begin{itemize}
        \item \textbf{\textcolor{rrtc}{RRT-C}:}
        RRT-C with collision checking.

        \item \textbf{\textcolor{cbfrrtc}{CBF-RRT-C}:}
        CBF-RRT-C variant with fixed steps derived from \cite{yang2019sampling}.

        \item \textbf{\textcolor{skipqp}{No-QP}:}
        Fixed-step accept-or-stop steering derived from
        \cite{yang2025lqrcbfrrt}. At each integration step, extension
        applies the nominal control unchanged if it satisfies every
        CBF constraint and otherwise terminates.
    \end{itemize}

\textbf{Ablations}
    \begin{itemize}
        \item \textbf{\textcolor{fixedstep}{Fixed}:}
        FLASH-CBF-RRT-C without horizon hops. It only uses the
        redundant QP constraint removal from
        Thm.~\ref{thm:cbf-row-redundancy}.

        \item \textbf{\textcolor{firstorder}{L1}:}
        FLASH-CBF-RRT-C with the first-order horizon.

        \item \textbf{\textcolor{secondorder}{No-QP + L2}:}
        The same accept-or-stop controller, as No-QP, utilizing adaptive horizon
        from Proposition.~\ref{prop:kinematic}. Accepted nominal control is held for the second-order certified horizon.
    \end{itemize}


\subsection{Planner Comparison}

Fig.~\ref{fig:mbm-planner-comparison} and
Table~\ref{tab:ur5_results} compare RRT-C, CBF-RRT-C, and QP-free CBF planners. RRT-C and CBF-RRT-C solve all 689 problems, while No-QP and No-QP+L2 solve 601 and 631, respectively. All CBF-based paths pass collision auditing, whereas one RRT-C path fails. Compared with RRT-C, CBF-RRT-C reduces the median tree size from 20 to 11 vertices and path cost from 13.69 to 9.86~rad, but increases planning time from 17.00 to 17.24~ms. Thus, its sample-efficiency gains are offset by QP overhead, consistent with~\cite{yang2025lqrcbfrrt}.

Our method reduces the median planning time to 4.45~ms, corresponding to speedups of \x{3.82} over RRT-C and \x{3.88} over CBF-RRT-C. It requires a median of 298 filter evaluations, compared with 3,597 collision checks for RRT-C and 1,119 filter evaluations for CBF-RRT-C. These represent reductions by factors of \x{12.07} and \x{3.76}, respectively, although collision checks and filter evaluations have different computational costs. The median tree size is 11 vertices, and the median path cost is 10.015~rad, approximately $1.5\%$ higher than CBF-RRT-C and $26.9\%$ lower than RRT-C.

The two QP-free variants isolate the benefit of adaptive
propagation from that of correcting the nominal control.
No-QP has 601 solutions, with medians of 92.69~ms, 90 vertices, and 10,639 evaluations. Adding the second-order time horizon increases the solution count
to 631 and reduces the medians to 12.24~ms, 99 vertices, and 682 evaluations. This results in a \x{7.57} speedup and \x{15.60} fewer evaluations.

Compared with No-QP+L2, our method reports 58 additional solutions
and is \x{2.75} faster in median planning time. It also reduces the median evaluation count by a factor of \x{2.29} and the median path cost by $44.14\%$, from 17.93 to 10.01~rad. These results indicate that both adaptive propagation and CBF steering improve planning performance in a complementary manner. The certified horizons reduce reevaluation frequency, while the CBF-QP enables sustained progress around obstacles, leading to longer extensions and improving sample efficiency and progress toward the goal.

\begin{figure}
    \centering
    \includegraphics[width=1\linewidth]{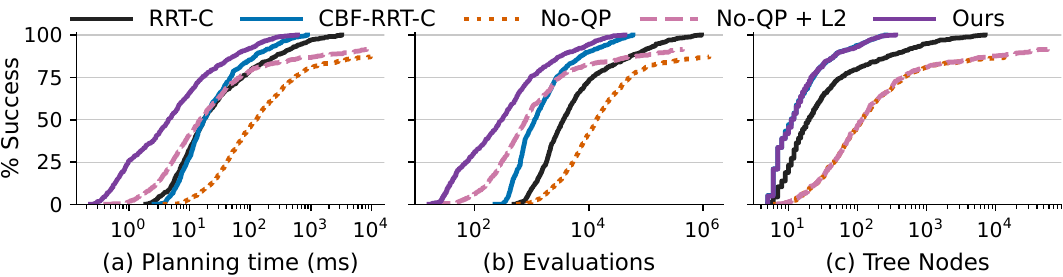}
    \caption{\small  Comparison with geometric, CBF-QP, and QP-free planning. Dark-gray, blue, orange, pink, and purple curves denote RRT-C, CBF-RRT-C, No-QP, No-QP + L2, and Ours respectively. Plots show the percentage of problems solved within the planning-time, filter-evaluation, or tree-node budget.}
    \label{fig:mbm-planner-comparison}
\end{figure}

\subsection{Adaptive-Horizon Ablation}
Fig.~\ref{fig:mbm-adaptive-ablation} compares the fixed-timestep, first-order, and second-order horizon variants of FLASH-CBF-RRT with identical constraint screening. This ablation shows how much each horizon bound affects planning. All three report successful solutions for all 689 problems, with a median of 11 vertices for all planners. The first-order horizon reduces the median filter evaluation count from 1,119 to 529 and the median planning time from 16.69 to 7.45~ms, corresponding to \x{2.12} fewer evaluations and a \x{2.24} speedup. The second-order horizon further reduces the evaluation count to 298 and planning time to 4.45~ms, yielding an additional \x{1.78} reduction in evaluations and \x{1.68} speedup. The second-order variant is \x{3.75} faster than the fixed step size variant in runtime, with \x{3.76} fewer CBF-QP solves. While the second-order variant reduces the CBF-QP solve count relative to the first-order variant by $43.67\%$, it reduces the runtime by $40.27\%$. This gap is due to the overhead of second-order certification.

\begin{figure}
    \centering
    \includegraphics[width=1\linewidth]{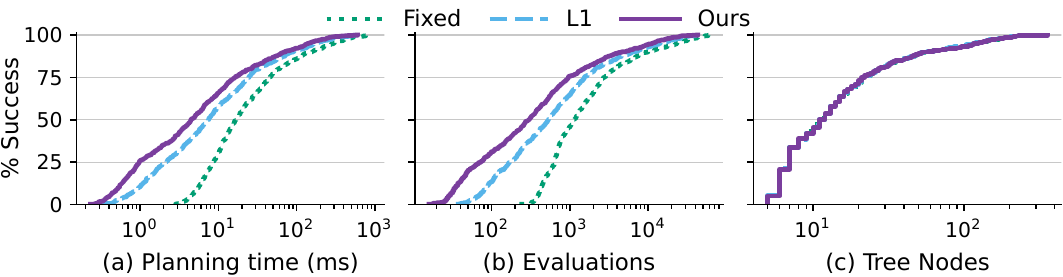}
    \caption{\small Adaptive-step ablation on the UR5 dataset. Green, turquoise, and purple denote Fixed, L1, and Ours, respectively. }
    \label{fig:mbm-adaptive-ablation}
\end{figure}

\input{Tables/Table_UR5Comparison}

\subsection{Real-Robot Planning Experiments}

We evaluate FLASH-CBF-RRT on a 17-DoF Pollen Reachy2 robot in a tabletop environment with obstacles as shown in Fig.~\ref{fig:reachyexp}. The SDF of the environment is estimated using Kernel-SDF \cite{kernelSDF} from RGB-D images from an Orbbec Gemini 336 RGB-D camera onboard the robot. All planning iterations were run onboard a SolidRun Bedrock V3000 industrial PC. Using the same map and the same start configuration, we drew 100 random valid goal configurations around the table. A ROS2~\cite{ros2} interface was implemented to query the SDF, and to control the robot.

\begin{figure}
    \centering
    \includegraphics[width=\linewidth,trim={0.5 1.70cm 0.5 0.5cm},clip]{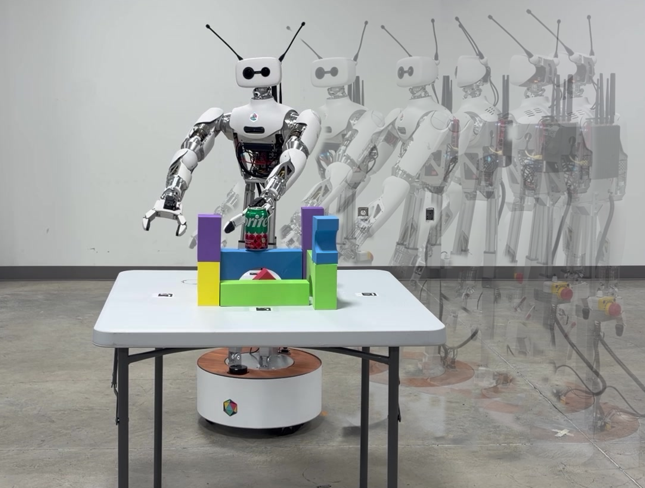}
    \caption{\small Reachy2 in the tabletop scene. Ghost poses show
    configurations along one planned 17-DoF motion, where the mobile base
    approaches the table, and both arms reach between the obstacles.}
    \label{fig:reachyexp}
\end{figure}
\input{Tables/TableReachyEXP}
Table~\ref{tab:reachy_results} reports median performance over 100 mobile-manipulator plans using identical instances and seeds. Both methods solve all goals, but FLASH-CBF-RRT is \(2.22\times\) faster, uses \(4.15\times\) fewer geometric evaluations, and produces \(15.1\%\) less arm motion, consistent with the UR5 results.

%% file: Tables/Table_UR5Comparison.tex
\begin{table}[t]
    \centering
    \caption{\small Results across 689 scenarios. Solved: number of scenarios solved. The reported statistics are medians. Eval.: number of collision checks (RRT-C), integration steps for No-QP, or CBF-QP iterations. Nodes: number of vertices in the tree; Motion: path length in radians; QP size: number of constraints per CBF-QP solve. CBF variants all return collision-free paths, while one collision was found among RRT-C. Best values are boldfaced.}
    \label{tab:ur5_results}
    \small
    \setlength{\tabcolsep}{4pt}
    \begin{tabular}{@{}lcccccc@{}}
        \toprule
        \textbf{Method} &
        \textbf{Solved} &
        \textbf{Time} &
        \textbf{Nodes} &
        \textbf{Eval.} &
        \textbf{Motion} &
        \textbf{QP} \\
        & & \textbf{(ms)} & & & \textbf{(rad)} & \textbf{size} \\
        \midrule
        \textcolor{ours}{Ours}
            & \textbf{689}
            & \textbf{4.45}
            & \textbf{11}
            & \textbf{298}
            & 10.01
            & 277.2 \\
        \hline
        \textit{Baselines} & & & & & & \\
        \textcolor{rrtc}{RRT-C}
            & \textbf{689}
            & 17.00
            & 20
            & 3597
            & 13.69
            & -- \\

        \textcolor{cbfrrtc}{CBF-RRT-C}~\cite{yang2019sampling}
            & \textbf{689}
            & 17.24
            & \textbf{11}
            & 1119
            & \textbf{9.86}
            & 342.0 \\

        \textcolor{skipqp}{No-QP}~\cite{yang2025lqrcbfrrt}
            & 601
            & 92.69
            & 90
            & 10639
            & 17.66
            & -- \\
        \hline
        \textit{Ablations} & & & & & & \\

        \textcolor{fixedstep}{Fixed}
            & \textbf{689}
            & 16.69
            & \textbf{11}
            & 1119
            & \textbf{9.86}
            & \textbf{266.2} \\

        \textcolor{firstorder}{L1}
            & \textbf{689}
            & 7.45
            & \textbf{11}
            & 529
            & 9.87
            & 275.7 \\
\textcolor{secondorder}{No-QP + L2}
            & 631
            & 12.24
            & 99
            & 682
            & 17.93
            & -- \\
        \bottomrule
    \end{tabular}
\end{table}

%% file: Tables/TableReachyEXP.tex
\begin{table}[!h]
    \centering
\caption{\small Median performance over 100 randomized onboard planning queries on a real Reachy2 mobile manipulator. Solved: no. of problems solved. Time: planning time. Eval: no. of collision checks or CBF-QP solves. Path length: base motion length. Motion: arm motion length.
Goals are sampled over a mapped table, with base poses placed on a
surrounding ring. Planning is performed in the full 17-DoF
configuration space using a live SDF map.}
    \label{tab:reachy_results}
    \small
    \setlength{\tabcolsep}{4pt}
    \begin{tabular}{@{}lcccccc@{}}
        \toprule
        \textbf{Method} &
        \textbf{Solved} &
        \textbf{Time} &
        \textbf{Nodes} &
        \textbf{Eval.} &
        \textbf{Path Length} &
        \textbf{Motion} \\
        & & \textbf{(ms)} & & & \textbf{(m)} & \textbf{(rad)} \\
        \midrule
        Ours
            & \textbf{100}
            & \textbf{12.21}
            & \textbf{10}
            & \textbf{206.5}
            & 1.648
            & \textbf{1.675} \\

        RRT-C
            & \textbf{100}
            & 27.05
            & 13
            & 856.5
            & \textbf{1.623}
            & 1.973 \\

        \bottomrule
    \end{tabular}
\end{table}

%% file: sections/conclusion.tex
\section{Conclusion}
\label{sec:conclusion}
We presented FLASH-CBF-RRT, a sampling-based planner that achieves fast, continuously safe motion planning by using a differentiable SDF representation in two complementary manners. 
It biases tree search toward free space using CBF-QP as a local collision-avoidance controller.
The resulting control action is applied with a certified adaptive step size determined using two novel closed-form bounds derived from the robot kinematics and the SDF clearance information. This allows FLASH-CBF-RRT to make long excursions in open spaces and short steps near obstacles. The bounds also allow removing provably inactive CBF constraints to reduce the QP size for faster computation. Experiments show that FLASH-CBF-RRT achieves both sample-efficient exploration and lower overall planning time, in both simulation and hardware. For future work, we plan to extend the adaptive horizon computation to time-varying fields and to explore asymptotically optimal planning.
